\documentclass[letterpaper]{article}
\usepackage{aaai16}
\usepackage{times}
\usepackage{helvet}
\usepackage{courier}

\usepackage{mathtools}
\usepackage{algpseudocode}
\usepackage{bbm}

\newcommand{\comment}[1]{}

\newcommand{\mA}{\mathcal{A}}

\newcommand{\mF}{\mathcal{F}}

\newcommand{\g}{\,|\,}

\newcommand{\1}{\mathbbm{1}}

\algblock{ParFor}{EndParFor}
\algnewcommand\algorithmicparfor{\textbf{parfor}}
\algnewcommand\algorithmicpardo{\textbf{do}}
\algnewcommand\algorithmicendparfor{\textbf{end\ parfor}}
\algrenewtext{ParFor}[1]{\algorithmicparfor\ #1\ \algorithmicpardo}
\algrenewtext{EndParFor}{\algorithmicendparfor}

\usepackage[utf8]{inputenc}
\usepackage[english]{babel}
\usepackage{amsthm}

\newtheorem{definition}{Definition}
\newtheorem{theorem}{Theorem}
\newtheorem{lemma}{Lemma}
\newtheorem{proposition}{Proposition}
\newtheorem{corollary}{Corollary}

\usepackage{multibib}
\newcites{latex}{References}

\title{Modeling Human Ad Hoc Coordination}

\author{Peter M. Krafft$^*$, Chris L. Baker$^\dagger$, Alex ``Sandy'' Pentland$^\ddagger$, Joshua B. Tenenbaum$^\dagger$\\Massachusetts Institute of Technology, Cambridge, MA USA\\$^*$Computer Science and Artificial Intelligence Laboratory, $^\dagger$MIT Media Lab, $^\ddagger$Department of Brain and Cognitive Sciences\\\{pkrafft,clbaker,pentland,jbt\}@mit.edu}

\usepackage{multicol}
\usepackage{graphicx}
\usepackage{algpseudocode}
\usepackage{algorithm}
\usepackage{mathtools}
\usepackage{algpseudocode}
\usepackage{bbm}
\usepackage{color}
\usepackage{soul}

\algtext*{EndFor}
\algtext*{EndWhile}
\algtext*{EndIf}

\begin{document}

\maketitle

\begin{abstract}
\begin{quote}
Whether in groups of humans or groups of computer agents,
collaboration is most effective between individuals who have the
ability to coordinate on a joint strategy for collective action.
However, in general a rational actor will only intend to coordinate if
that actor believes the other group members have the same intention.
This circular dependence makes rational coordination difficult in
uncertain environments if communication between actors is unreliable
and no prior agreements have been made.  An important normative
question with regard to coordination in these ad hoc settings is
therefore how one can come to believe that other actors will
coordinate, and with regard to systems involving humans, an important
empirical question is how humans arrive at these expectations.  We
introduce an exact algorithm for computing the infinitely recursive
hierarchy of graded beliefs required for rational coordination in
uncertain environments, and we introduce a novel mechanism for
multiagent coordination that uses it.  Our algorithm is valid in any
environment with a finite state space, and extensions to certain
countably infinite state spaces are likely possible.  We test our
mechanism for multiagent coordination as a model for human decisions
in a simple coordination game using existing experimental data.  We
then explore via simulations whether modeling humans in this way may
improve human-agent collaboration.
\end{quote}
\end{abstract}

Forming shared plans that support mutually beneficial behavior within
a group is central to collaborative social interaction and collective
intelligence \citelatex{grosz1996collaborative}.  Indeed, many common
organizational practices are designed to facilitate shared knowledge
of the structure and goals of organizations, as well as mutual
recognition of the roles that individuals in the organizations play.
Once teams become physically separated and responsiveness or frequency
of communication declines, the challenge of forming shared plans
increases.  Part of this difficulty is fundamentally computational.
In theory, coming to a fully mutually recognized agreement on even a
simple action plan among two choices can be literally impossible if
communication is even mildly unreliable, even if an arbitrary amount
of communication is allowed
\citelatex{halpern1990knowledge,lynch1996distributed}.

This problem is well-studied within the AI literature (e.g.,
\citelatex{gmytrasiewicz1992decision}), though the core difficulties still
manifest in contemporary research on ``ad hoc
coordination''---collaborative multiagent planning with previously
unknown teammates \citelatex{stone2010ad}.  However, surprisingly little is
known about the strategies that \textit{humans} use to overcome the
difficulties of coordination \citelatex{thomas2014psychology}.
Understanding how and when people try to coordinate is critical to
furthering our understanding of human group behavior, as well as to
the design of agents for human-agent collectives
\citelatex{jennings2014human}.  Existing attempts at modeling human
coordination have focused either on unstructured predictive models
(e.g., \citelatex{frieder2012agent}) or bounded depth socially recursive
reasoning models (e.g., \citelatex{gal2008networks,yoshida2008game}), but
there is reason to believe that these accounts miss important aspects
of human coordination.







One concept that appears repeatedly in formal treatments of
coordination but has not appeared meaningfully in empirical modeling
is common knowledge.  Two agents have common knowledge if both agents
have infinitely nested knowledge of the other agent's knowledge of a
proposition, i.e. the first agent knows the second agent knows, the
first agent knows the second agent knows the first agent knows, etc.
Common knowledge has been shown to be necessary for exact coordination
\citelatex{halpern1990knowledge}, and a probabilistic generalization of
common knowledge, called common p-belief, has been been shown to be
necessary for approximate coordination
\citelatex{monderer1989approximating}.  While these notions are clearly
important normatively, it is not entirely clear how important they are
empirically in human coordination.  Indeed, supposing that humans are
able to mentally represent an infinitely recursive belief state seems
a priori implausible, and the need to represent and infer this
infinite recursive belief state has also been a barrier to empirically
testing models involving common knowledge.

Nevertheless, building on the existing normative results, a group of
researchers recently designed a set of experiments to test whether
people are able to recognize situations in which common knowledge
might obtain \citelatex{thomas2014psychology} (hereafter referred to as the
``Thomas experiments'').  These researchers argued that people do
possess a distinct mental representation of common knowledge by
showing that people will attempt to coordinate more often in
situations where common knowledge can be inferred.  However, this
previous work did not formalize this claim in a model or rigorously
test it against plausible alternative computational models of
coordination.  This existing empirical work therefore leaves open
several important scientific questions that a modeling effort can help
address.  In particular: How might people mentally represent common
p-belief?  Do people reason about graded levels of common p-belief, or
just ``sufficiently high'' common p-belief?  Finally, what
computational processes could people use to infer common p-belief?





In this work we use a previously established fixed point
characterization of common p-belief \citelatex{monderer1989approximating}
to formulate a novel model of human coordination.  In finite state
spaces this characterization yields an exact finite representation of
common p-belief, which we use to develop an efficient algorithm for
computing common p-belief.  This algorithm allows us to simulate
models that rely on common p-belief.  Because of the normative
importance of common p-belief in coordination problems, our algorithm
may also be independently useful for coordination in artificial
multiagent systems.  We show using data from the Thomas experiments
that this model provides a better account of human decisions than
three alternative models in a simple coordination task.  Finally, we
show via simulations based on the data from the Thomas experiments
that modeling humans in this way may improve human-agent coordination.





\section{Background}

We first provide a description of the coordination task we will study
in this paper: the well-known coordinated attack problem.  We
then provide an overview of the formal definitions of common knowledge
and common p-belief, and their relationship to the coordinated
attack problem.

\subsection{Coordinated Attack Problem}

The coordination task that we study in this paper is alternatively
called the coordinated attack problem, the two generals problem, or
the email game.  The original formulation of this task was posed in
the literature on distributed computer systems to illustrate the
impossibility of achieving consensus among distributed computer
processors that use an unreliable message-passing system
\citelatex{halpern1990knowledge}, and the problem was later adapted by
economists to a game theoretic context
\citelatex{rubinstein1989electronic}.  Here we focus on the game theoretic
adaptation, as this formulation is more amenable to decision-theoretic
modeling and thus more relevant for modeling human
behavior.\footnote{Our exposition largely assumes familiarity with
  rudimentary game theory, and familiarity with measure-theoretic
  probability as it appears in incomplete information games.}

In this task the world can be in one of two states, $x = 1$ or $x =
0$.  The state of the world determines which of two games two players
will play together.  The payoff matrices for these two games are as
follows ($a > c > \max(b,d)$):
\begin{center}
  \begin{tabular}{|l|l|l|}
    \hline
        {\bf $x = 1$} & {\bf $A$} & {\bf $B$} \\ \hline
        $A$             & $a$,$a$     & $b$,$c$     \\ \hline
        $B$             & $c$,$b$     & $c$,$c$     \\ \hline
  \end{tabular}\hspace{0.2cm}
  \begin{tabular}{|l|l|l|}
    \hline
        {\bf $x = 0$} & {\bf $A$} & {\bf $B$} \\ \hline
        $A$             & $d$,$d$     & $b$,$c$     \\ \hline
        $B$             & $c$,$b$     & $c$,$c$     \\ \hline
  \end{tabular}
\end{center}
The players receive the optimal payoff if they coordinate on both
playing $A$ when $x = 1$, but playing $A$ is risky.  Playing $A$ is
inferior to playing $B$ if $x = 0$ or if there is a mismatch between
the players' actions.  Playing $B$ is safe with a sure payoff of $c$.
Thus in order for it to be rational to play $A$, a player must believe
with sufficient confidence both that the world state is $x = 1$ and
that the other player will play $A$.





\subsection{Common p-Belief}

In order for a player to believe the other player will play $A$ in
this game, it is not enough for that player to believe that the other
player knows $x = 1$.  If the second player does not believe that the
first player knows $x = 1$, then the second player will not try to
coordinate.  Therefore the first player must also at least believe
that the second player believes the first player knows $x = 1$.
However, it turns out even this amount of knowledge does not suffice.
In fact, an infinite hierarchy of recursive belief is needed
\citelatex{morris1997approximate}.  This infinite hierarchy of beliefs has
been formalized using a construct called common p-belief, which we now
define.

Using standard definitions from game theory, we define a two-player
finite Bayesian game to be a tuple $((\Omega, \mu, (\Pi_0, \Pi_1)),
(\mA_0, \mA_1), (u_0, u_1))$ consisting of a finite state space
$\Omega = \{\omega_1, \ldots, \omega_{|\Omega|}\}$, a probability
measure $\mu$ defined over that state space, the information partition
$\Pi_i$ of each player $i$, the action set $\mA_i$ of each player, and
the utility function $u_i$ of each player.  Elements of $\Omega$ are
called states, and subsets of $\Omega$ are called events.  For a given
world state $\omega$ and player $i$ the partition of $\Omega$,
$\Pi_i$, uniquely specifies the beliefs of player $i$ in the form of
posterior probabilities.  $\Pi_i(\omega)$, which indicates the unique
element of $\Pi_i$ containing $\omega$, can be thought of as the
observation that player $i$ receives when the true state $\omega$
occurs.  Specifically for any event $E \subseteq \Omega$, $\mu(E \g
\Pi_i(\omega))$ is the probability that player $i$ assigns to $E$
having occurred given $\omega$ has occurred.  As a shorthand, we write
$P_i(E \g \omega) = \mu(E \g \Pi_i(\omega))$.  Using another common
shorthand, we will treat propositions and events satisfying those
propositions interchangeably.  For example, in the coordinated attack
problem we will be interested in whether there is ``common p-belief''
that $x = 1$, which will refer to common p-belief in the event $C =
\{\omega \in \Omega : x(\omega) = 1\}$, where $x$ formally is a random
variable mapping $\Omega \rightarrow \{0,1\}$.

Following \citelatex{dalkiran2012common}, we say that player $i$
$p$-believes\footnote{We use an italicized ``$p$'' when referring to
  specific values of p-belief and a non-italicized ``p'' when
  referring to the terms in general.} an event $E$ at $\omega$ if
$P_i(E \g \omega) \geq p$.  An event $E$ is said to be $p$-evident if
for all $\omega \in E$ and for all players $i$, player $i$
$p$-believes $E$ at $\omega$.  In a slight divergence from the
standard terminology of this literature, we say an event $E$ is
super-$p$-evident if for all $\omega \in E$ and for all players $i$,
$P_i(E \g \omega) > p$ (the only difference being strict inequality).
We say there is common $p$-belief in an event $C$ at state $\omega$ if
there exists a $p$-evident event $E$ with $\omega \in E$, and for all
$\omega' \in E$ and all players $i$, player $i$ $p$-believes $C$ at
$\omega'$.  Common knowledge is defined as common $p$-belief for $p
=1$.

A critically important result of \citelatex{monderer1989approximating}
states that this definition of common p-belief is equivalent to a more
intuitive infinitely recursive formulation.
The importance of this definition of common p-belief is therefore that
it provides a fixed point characterization of common p-belief strictly
in terms of beliefs about events rather than directly in terms of
beliefs about other players.  When $\Omega$ is finite, common p-belief
can thus be represented in terms of a finite set of states, rather
than an infinite hierarchy of beliefs.




\section{Models}

We now describe four strategies for coordination in the coordinated
attack game we study.  Two of these strategies involve the computation
of p-evident events and common p-belief, which we will use to test
whether human coordination behavior could be explained in terms of
reasoning about p-evident events.  The other two strategies serve as
baselines.

\subsection{Rational p-Belief}

The first strategy we consider represents an agent who maximizes
expected utility at an equilibrium point of the coordinated attack
problem we study.  The strategy is implemented as follows: player $i$
plays action $A$ if and only if the player believes with probability
at least $p^* = \frac{c - b}{a - b}$ that both players have common
$p^*$-belief that $x = 1$.  This strategy forms an equilibrium of the
coordinated attack problem if $p^* > P_i(x = 1)$ and if evidence that
$x = 1$ always leads to certain belief that $x = 1$, i.e. $P_i(x = 1
\g \omega) > P_i(x = 1) \Rightarrow P_i(x = 1 \g \omega) = 1$ for all
$\omega$.  These conditions will be satisfied by the specific state
spaces and payoffs we use to represent the Thomas experiments.  We
call this model the rational p-belief strategy.  The proof that this
strategy forms an equilibrium, including a derivation for the specific
form of $p^*$, is included in our supplementary materials.\footnote{A
version of our paper that includes the supplementary materials (as
well as any post-publication corrections to the main text) is
available in multiple locations online, including at
http://people.csail.mit.edu/pkrafft/papers/krafft-et-al-2016-modeling-human-ad-hoc-coordination.pdf.}

\subsection{Matched p-Belief}

The second strategy we consider is a novel probabilistic relaxation of
the rational p-belief strategy.  Humans have been shown to exhibit a
behavior called probability matching in many decision-making settings
\citelatex{herrnstein1961relative,vulkan2000economist}.  Probability
matching consists of taking the probability $p$ that a decision is the
best decision available, and choosing to make that decision with
probability $p$.  While probability matching is not utility
maximizing, it can be viewed as rational if players are performing
sample-based Bayesian inference and if taking samples is costly
\citelatex{vul2014one}.  Motivated by this frequently observed behavior, we
propose a model we call the matched p-belief strategy.  A player $i$
using this strategy chooses action $A$ at $\omega$ with probability
$p$ equal to the maximal common p-belief that the player perceives at
$\omega$, i.e. the largest value such that $i$ $p$-believes at
$\omega$ that there is common $p$-belief that $x = 1$.



\subsection{Iterated Maximization}

Next we consider a well-known model of boundedly rational behavior
sometimes called a ``level-$k$'' depth of reasoning model.  This
family of models has been shown to be consistent with human behavior
in a diversity of settings, including some coordination games (e.g.,
\citelatex{yoshida2008game}), and hence is a strong baseline.  Since the
term ``level-$k$'' is used for many slightly different models, we call
our instantiation of this model the iterated maximization strategy.
This strategy assumes that players have a certain fixed level of
recursive social reasoning $k$.  A player using the level-$k$ iterated
maximization strategy chooses the action that maximizes that player's
expected utility when playing with a player using the level-$(k-1)$
strategy.  The level-$0$ player takes action $A$ at $\omega$ if $P_i(x
= 1 \g \omega) > \frac{c - b}{a - b}$.  This level-0 strategy
corresponds to the player maximizing expected utility assuming the
player can control the actions of both players, or equivalently that
the optimal joint action is taken according to that player's beliefs.
While in general the predictions of level-$k$ models depend strongly
on the specification of the level-0 strategy, in informal exploration
we found that the qualitative conclusions of our work are robust to
whether we instead specify the level-0 strategy as always playing $A$
or choosing between $A$ and $B$ uniformly randomly.


\subsection{Iterated Matching}

Finally, we also consider a less common depth of reasoning model that
combines the iterated maximization strategy with probability matching
behavior, which we call iterated matching.  Like the iterated
maximization strategy, this strategy assumes that players have a
certain fixed level of recursive social reasoning $k$.  However,
instead of choosing the action that maximizes expected utility, a
level-$k$ player using the iterated matching strategy chooses to take
action $A$ with probability equal to that player's belief that $x =
1$, times the expected probability that a level-$(k-1)$ companion
player would play $A$.  The level-$0$ player probability matches on
$P_i(x = 1 \g \omega)$.


\section{Algorithms}

In this section we present the algorithms we use to implement each of
the models we consider.  To the best of our knowledge the existing
literature on common p-belief has yet to offer algorithms for
computing common p-belief (or in our case the perceived maximal common
p-belief) for a given world state and observation model.  This
computation is central to the rational and matched p-belief
strategies.  Hence we offer the first fully computational account of
coordination via reasoning about p-evident events.  Algorithms for
iterated reasoning are straightforward and well-known.


The challenge in developing an algorithm for computing a player's
perception of the maximal common p-belief is avoiding enumeration over
all exponentially many possible subsets of $\Omega$.  While it is
straightforward to evaluate whether a particular given event is
$p$-evident, the definition of common $p$-belief requires only the
existence of some such event.  Computing perceived maximal common
p-belief therefore requires jointly searching over values of $p$ and
over subsets of $\Omega$.  We leverage the generic mathematical
structure of p-evident events in finite state spaces in order to
develop an exact algorithm that avoids enumerating all subsets.  Our
algorithm only requires a search that is polynomial in the size of the
state space.  Of course, the state spaces in many problems are often
themselves exponential in some underlying parameter variables, and
hence future improvements on this algorithm would be desirable.
Extensions to at least certain countably infinite or continuous state
spaces are likely possible as well, such as perhaps by refining the
search to only consider events that have non-zero probability given
the player's observations.
 
\subsection{Computing Information Partitions}

Our algorithms require access to the information partitions of each
player.  However, directly specifying the information partitions that
people have in a naturalistic setting, such as in the data we use from
the Thomas experiments, is difficult.  Instead, we take the approach
of specifying a plausible generative probabilistic world model, and we
then construct the information partitions from this factored
representation via a straightforward algorithm.  The generative world
model specifies the pieces of information, or ``observations'', that
each player receives.  The algorithm for generating information
partitions, which is specified precisely in our supplementary
materials, consists of iterating over all combinations of random
values of variables in the generative world model, treating each such
combination as a state in $\Omega$, and for each player grouping
together the states that yield identical observations.

\subsection{Computing Common p-Belief}

\begin{algorithm}[t]
  \caption{$\text{common\_p\_belief}(C, i, \omega)$}
  \label{alg:pbelief}
\begin{algorithmic}
\State $E := \Omega$; $F := \Omega$
\While{$P_i(F \g \omega) > 0$}        
  \State $p := \text{evidence\_level}(F, C)$ 
  \State $E := F$
  \State $F := \text{super\_p\_evident}(E, C, p)$
\EndWhile
\State $p := \text{evidence\_level}(E, C)$
\State \Return $p$
\end{algorithmic}
\end{algorithm}

\begin{algorithm}[t]
  \caption{evidence\_level($E$, $C$)}
  \label{alg:evidencelevel}
\begin{algorithmic}
\State \Return $\min_{\omega \in E} \text{min\_belief}(E,C,\omega)$
\end{algorithmic}
\end{algorithm}

\begin{algorithm}[t]
  \caption{min\_belief($E$, $C$, $\omega$)}
  \label{alg:minbelief}
\begin{algorithmic}
\State \Return $\min_{i \in \{0,1\}} \min(P_i(E \g \omega), P_i(C \g \omega))$
\end{algorithmic}
\end{algorithm}

\begin{algorithm}[t]
  \caption{super\_p\_evident($E$, $C$, $p$)}
  \label{alg:evident}
\begin{algorithmic}
\While{$E$ has changed}
\If{$\exists\ \omega \in E : \text{min\_belief}(E,C,\omega) \leq p$}
\State $E := E \setminus \{\omega\}$
\EndIf
\EndWhile  
\State \Return $E$
\end{algorithmic}
\end{algorithm}


Algorithms 1-4 present the functions needed to compute perceived
maximal common p-belief.  Formally, given a player $i$, a particular
state $\omega$, and an event $C$, Algorithm 1 computes the largest
value $p$ for which player $i$ $p$-believes that there is common
$p$-belief in $C$.  Note that it is insufficient to compute the
largest value of $p$ for which there is common $p$-belief in $C$ at
$\omega$, since in general at state $\omega$ player $i$ only knows
that the event $\Pi_i(\omega)$ has occurred, not that $\omega$
specifically has occurred.  Relatedly, note that while Algorithm 1
takes $\omega$ as input for convenience, the algorithm only depends on
$\Pi_i(\omega)$, and hence could be executed by a player.

Formal proofs of the correctness of these algorithms are included in
our supplementary materials.  The basic logic of the algorithms is to
maintain a candidate p-evident event $E$, and to gradually remove
elements from $E$ to make it more p-evident until a point where player
$i$ believes the event to be impossible because the elements of the
event are no longer consistent with that player's observations.  By
only removing elements that either cause $E$ to be unlikely or cause
$C$ to be unlikely, we are guaranteed to arrive at a more p-evident
event at each iteration, and one that preserves belief in $C$.  By
starting with $E$ as the entire state space, the final candidate event
must be the largest, most p-evident event that player $i$ $p$-believes
at state $\omega$ in which $C$ is common $p$-belief.  This algorithm
can also be viewed as traversing a unique nested sequence of maximally
evident events (independent of $i$ and $\omega$) induced by $C$ and
the structure of $(\Omega, \mu, (\Pi_0, \Pi_1))$, halting at the first
event in this sequence that does not include any elements of
$\Pi_i(\omega)$.

The rational p-belief strategy consists of player $i$ taking action
$A$ if common\_p\_belief($x = 1$, $i$, $\omega$) $> \frac{c -
  b}{a - b}$.  The matched p-belief strategy consists of player $i$
choosing $A$ with probability common\_p\_belief($x = 1$, $i$,
$\omega$).

\subsection{Iterated Reasoning}

We now present our algorithms for the iterated reasoning strategies.
For level $k > 0$, given a player $i$ and a state $\omega$, the
iterated maximization strategy computes
\begin{align*}
  &f_i^k(\omega) = \1
  \Big(
  \sum_{\omega' \in \Pi_i(\omega)}
  P_i(\omega' \g \omega)
  \Big(
  P_i(x = 1 \g \omega') f_{1 - i}^{k-1}(\omega') a 
  \\
  & + P_i(x = 0 \g \omega') f_{1-i}^{k-1}(\omega') d
  + \big(1 - f_{1 - i}^{k-1}(\omega')\big) b
  \Big) > c
  \Big),
\end{align*}
where $\1()$ is the indicator function, and $f_i^0(\omega) = \1(P_i(x
= 1 \g \omega) > \frac{c - b}{a - b}$).  If $f_i^k(\omega) = 1$,
then player $i$ plays $A$, and otherwise player $i$ plays $B$.  For
the iterated matching strategy,
\begin{align*}
  q_i^{k}(\omega) = P_i(x = 1 \g \omega ) \cdot \sum_{\omega' \in \Pi_i(\omega)} P_i(\omega' \g \omega)
  q_{1 - i}^{k-1}(\omega').
\end{align*}
$A$ is then played with probability $q_i^k$, $q_i^0 = P_i(x = 1 \g
\omega)$.

\section{Data}

\begin{figure*}
  \centering
  \includegraphics[width = 0.20\linewidth]{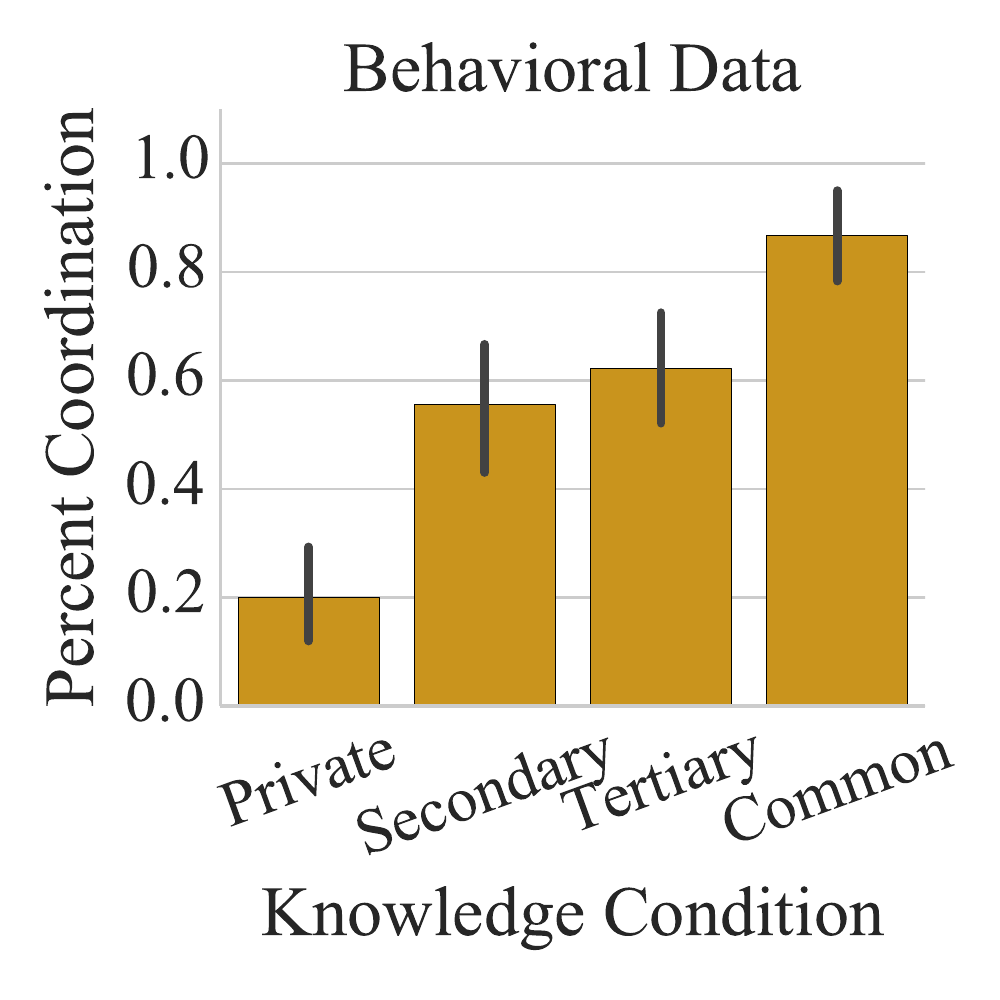}
  \includegraphics[width = 0.79\linewidth]{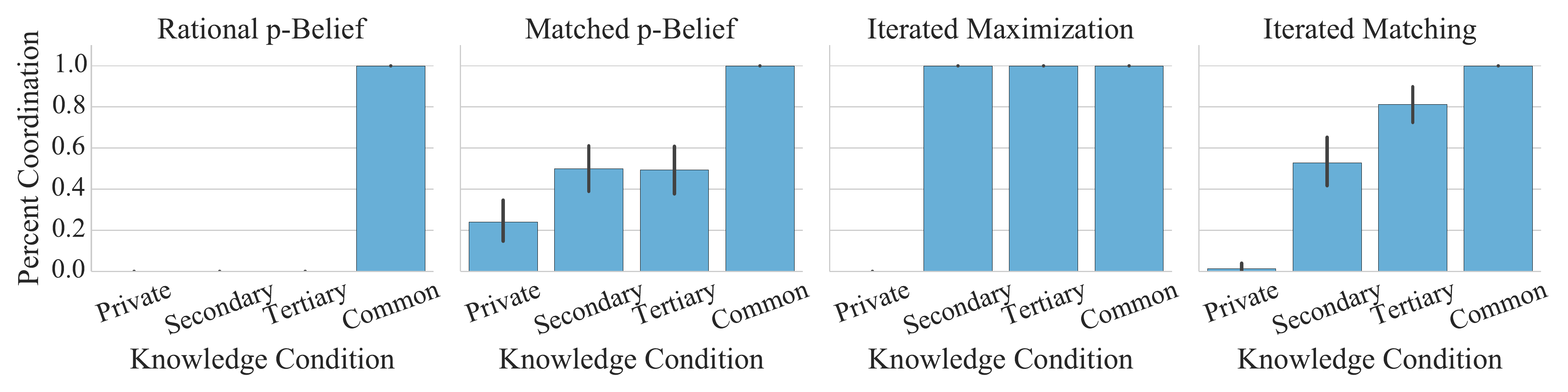}
  \caption{Data from the Thomas experiments and the predictions of
    each of the models we consider.}
  \label{fig:data}
\end{figure*}


We now present the data we use for our empirical results.  The dataset
comes from the Thomas experiments \citelatex{thomas2014psychology}.  These
experiments presented participants with a stylized coordinated attack
problem couched in a story about the butcher and the baker of a town.
In their story, these merchants can either work together to produce
hot dogs, or they can work separately to produce chicken wings and
dinner rolls, respectively.  The merchants can sell chicken wings and
dinner rolls separately for a constant profit of $c$ each on any day,
but the profit of hot dogs varies from day-to-day.  The merchants make
a profit of $a$ each if $x = 1$ on a particular day or $d$ if $x = 0$.
There is also a loudspeaker that sometimes publicly announces the
prices of hot dogs, and a messenger who runs around the town
delivering messages.  The experiments had four different knowledge
conditions that specified the information that participants received:

\begin{enumerate}
\item \textbf{Private Knowledge:} ``The Messenger Boy [sic] has not seen the
  Butcher today, so he cannot tell you anything about what the Butcher
  knows.''

\item \textbf{Secondary Knowledge:} ``The Messenger Boy says he
  stopped by the butcher shop before coming to your bakery. He tells
  you that the Butcher knows what today’s hot dog price is. However,
  he says that he forgot to mention to the Butcher that he was coming
  to see you, so the Butcher is not aware that you know today’s hot
  dog price.''

\item \textbf{Tertiary Knowledge:} ``The Messenger Boy mentions that
  he is heading over to the butcher shop, and will let the Butcher
  know today’s price as well. The Messenger Boy will also tell the
  Butcher that he just came from your bakery and told you the
  price. However, the Messenger Boy will not inform the Butcher that
  he told you he would be heading over there. So, while the Butcher is
  aware that you know today’s price, he is not aware that you know
  that he knows that.''

\item \textbf{Common Knowledge:} ``The loudspeaker broadcast the
  market price . . . The messenger boy did not come by. Because the
  market price was broadcast on the loudspeaker, the Butcher knows
  today's price, and he knows that you know this information as
  well.''
\end{enumerate}

After being shown this information as well as additional information
indicating that $x = 1$, the participants were asked whether they
would like to try to make hot dogs or not.  The dataset from this
experiment is visualized in Figure \ref{fig:data}.  Since the
researchers provided evidence that the behavior of participants in
their two-player experiments was invariant to payoffs, here we focus
on their first payoff condition, in which $a = 1.1$, $b = 0$, $c = 1$,
and $d = 0.4$.

We use this dataset to test whether the coordination strategies we
have described are good models of human coordination in this setting.
In order to be able to generate predictions for these models, we must
determine a state space that represents the story in the Thomas
experiments.
We designed the following two probabilistic generative world models
(one for the messenger, and one for the loudspeaker) to be consistent
with a reading of the knowledge conditions from those experiments.
The observe($i$,$o$) function indicates that player $i$ observes $o$.
\begin{align*}
  &\textbf{Messenger:}\\
  &\phantom{xxxx}x \sim \text{Bernoulli}(\delta)
  \\
  &\phantom{xxxx}\text{visit}_0 \sim \text{Bernoulli}(0.5)
  \\
  &\phantom{xxxx}\text{visit}_1 \sim \text{Bernoulli}(0.5)
  \\
  &\phantom{xxxx}\text{tell\_plan}_0 \sim \text{visit}_0 \land \text{Bernoulli}(0.5)
  \\
  &\phantom{xxxx}\text{tell\_plan}_1 \sim \text{visit}_1 \land \text{Bernoulli}(0.5)
  \\
  &\phantom{xxxx}\text{\textbf{if} visit$_0$:}
  \\
  &\phantom{xxxx}\phantom{xxx} \text{observe}(0, x)
  \\
  &\phantom{xxxx}\phantom{xxx} \text{\textbf{if} tell\_plan$_0$:}
  \\
  &\phantom{xxxx}\phantom{xxx} \phantom{xxx} \text{observe}(0, (\text{visit}_1, \text{tell\_plan}_1))
  \\
  &\phantom{xxxx}\text{\textbf{if} visit$_1$:}
  \\
  &\phantom{xxxx}\phantom{xxx} \text{observe}(1, (x, \text{visit}_0))
  \\
  &\phantom{xxxx}\phantom{xxx} \text{\textbf{if} tell\_plan$_1$:}
  \\
  &\phantom{xxxx}\phantom{xxx} \phantom{xxx} \text{observe}(1, \text{tell\_plan}_0)\\
&\textbf{Loudspeaker:}\\
  &\phantom{xxxx}x \sim \text{Bernoulli}(\delta)
  \\
  &\phantom{xxxx}\text{broadcast} \sim \text{Bernoulli}(0.5)
  \\
  &\phantom{xxxx}\text{\textbf{if} broadcast:}
  \\
  &\phantom{xxxx}\phantom{xxx}\text{observe}(0, x), \text{observe}(1, x)
\end{align*}

These models share a free parameter $\delta$.  We take $\delta =
0.25$. This setting provides a closer fit to the empirical data than
the maximum entropy setting of $\delta = 0.5$.  We interpret
statements that one player is ``not aware'' as meaning that the player
could have been made aware, and assign a maximum entropy probability
of 0.5 to these events.

The state spaces corresponding to these world models consist of the
sets of all possible combinations of variables in the models'
generative processes: ($x$, visit$_0$, visit$_1$, tell\_plan$_0$,
tell\_plan$_0$) for $\Omega_{\text{messenger}}$ and ($x$, broadcast)
for $\Omega_{\text{loudspeaker}}$.  The generative processes also
uniquely specify probability measures over each state space.  The
knowledge conditions correspond to the following states.  Private:
$(1,1,0,1,0) \in \Omega_{\text{messenger}}$, Secondary: $(1,1,1,0,1)
\in \Omega_{\text{messenger}}$, Tertiary: $(1,1,1,1,0) \in
\Omega_{\text{messenger}}$, and Common Knowledge: $(1, 1) \in
\Omega_{\text{loudspeaker}}$.  The participants act as player 0 in all
but the secondary condition.  Due to high ambiguity in the wording of
the private knowledge condition, we considered two plausible readings.
Either the messenger is communicating an intention to not visit the
other player, or the messenger is being unhelpful in not offering any
information about the messenger's plan.  By using the state
$(1,1,0,1,0)$ we assume the first interpretation.  This interpretation
results in a better empirical fit.





\section{Results}

\begin{figure}
  \centering
  \includegraphics[width = 0.7\linewidth]{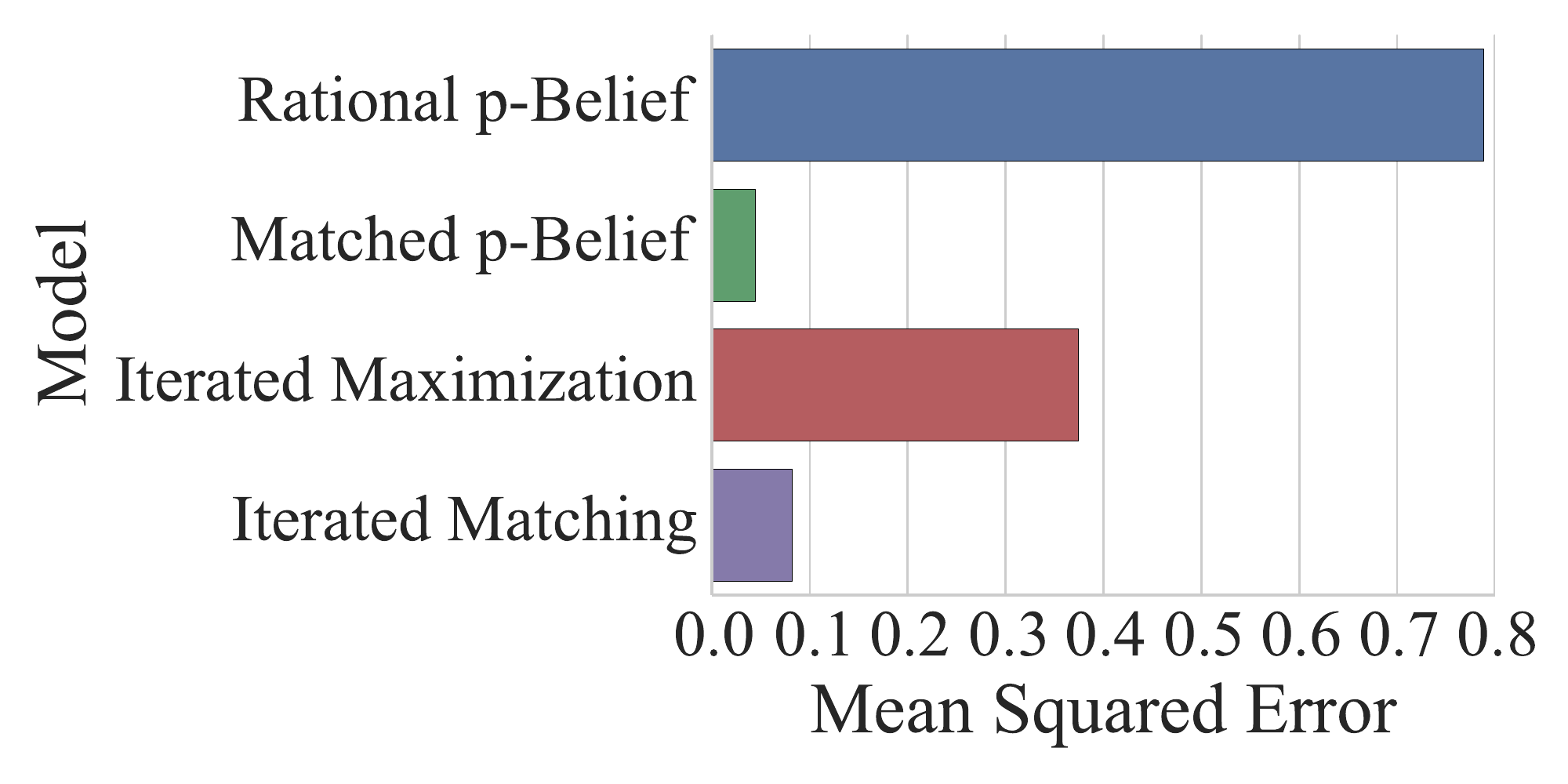}
  \caption{The mean-squared error of each model's predictions on the Thomas experiments' data.}
  \label{fig:error}
\end{figure}

We now present our empirical results.  We first examine the
predictions of each of the coordination strategies we consider given
the generative processes representing the Thomas experiments.  We then
examine the extent to which a computer agent equipped with the best
fitting model of human coordination is able to achieve higher payoffs
in a simulated human-agent coordination problem.  All of our code is
available online at
https://github.com/pkrafft/modeling-human-ad-hoc-coordination.
  
\subsection{Model Comparison}

To perform model comparison we compute the probability of choosing $A$
that each model predicts given our formal representations of each of
the four knowledge conditions.  We then compare these predicted
probabilities to the actual probabilities observed in the Thomas
experiments.  For the two iterated reasoning models we use a grid
search over $[0,1,2,3,4,5]$ to find the best fitting $k$ for each
model (ultimately $k = 1$ in iterated maximization and $k = 3$ in
iterated matching).  The specific predictions of each model are shown
in Figure \ref{fig:data}.  As shown in Figure \ref{fig:error}, the
matched p-belief model achieves the lowest mean-squared error.
Qualitatively, the most striking aspect of the data that the matched
p-belief model successfully captures is the similarity in the
probability of coordination between the secondary and tertiary
knowledge conditions.  The two models that involve maximizing agents
(rational p-belief and iterated maximization) both make predictions
that are too extreme.  The iterated matching model offers a
competitive second place fit to the data, but it fails to capture the
similarity between the middle two knowledge conditions.

The reason that the matched p-belief model makes good predictions for
the two middle conditions is that the player in both of those
conditions has the same amount of uncertainty appearing at some level
of that player's infinite recursive hierarchy of interpersonal
beliefs.  Common p-belief essentially represents a minimum taken over
all these levels, and thus the common p-belief in each of those two
conditions is the same.
The rational p-belief model is aware of the uncertainty at higher
levels of recursive belief, but its predictions are too coarse due to
the assumption of utility maximization.  An interesting avenue for
future work is to examine whether the rational p-belief model can be
relaxed in any other way to allow for intermediate predictions, such
as by allowing for heterogeneity in interpretations of the world model
across agents.  It is possible that the matched p-belief model is
approximating such a case.

\subsection{Human-Agent Coordination}

\begin{figure}
  \centering
  \includegraphics[width = 0.9\linewidth]{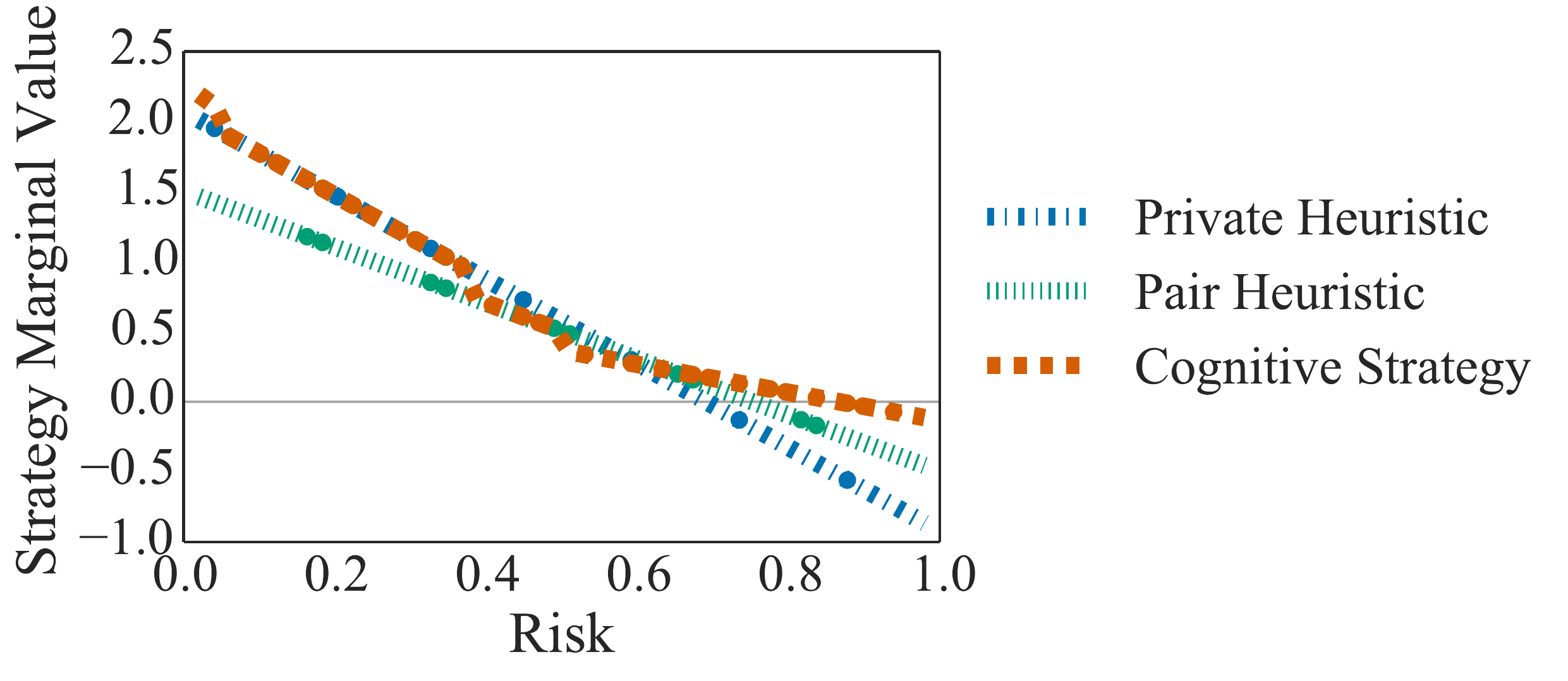}
  \caption{Performance of agents in our simulated human-agent
    coordination experiments.  A strategy's marginal value is the
    expected sum of payoffs the strategy obtained in each of the four
    knowledge conditions, minus the payoffs that could have been
    obtained by always playing $B$.}
  \label{fig:simulations}
\end{figure}

Besides testing the fit of the models of human coordination that we
have proposed, we are also interested in whether the best fitting
model helps us improve outcomes in human-agent coordination.  We use
the data from the Thomas experiments to evaluate this possibility.
For this task our computer agents implement what we call a ``cognitive
strategy''.  An agent using the cognitive strategy chooses the action
that maximizes expected utility under an assumption that the agent's
companion is using the matched p-belief model.  These agents play the
humans' companion player in each of the four knowledge conditions of
the Thomas experiments (player 1 in all but the Secondary condition).
We evaluate the payoffs from the agents' actions using the human data
from the Thomas experiments.  In this simulation we vary the payoffs
$(a,b,c,d)$, and we assume that the humans would remain payoff
invariant across the range of payoffs that we use.  This assumption is
reasonable given that participants in the Thomas experiments displayed
payoff invariance in the two-agent case.  We vary the payoffs
according to risk, taking the payoffs to be $(1, 0, p^*, 0)$ for a
particular risk level $p^* = \frac{c - b}{a - b}$.  As shown in
Figure \ref{fig:simulations}, we find that having the matched p-belief
model of human coordination may help in human-agent coordination.  We
compared to two baseline strategies: an agent using a ``private
heuristic'' who always coordinates if the agent knows $x = 1$, and an
agent using a ``pair heuristic'' who always coordinates if the agent
knows that both the agent and the human know $x = 1$.  The private
heuristic achieves good performance for low risk levels, and the pair
heuristic achieves good performance for high risk levels.  The
cognitive strategy achieves good performance at both low and high
levels of risk, and only has negative marginal value over always
playing the safe action B at very high levels of risk.




\section{Discussion}

In the present paper we focused on laying the groundwork for using
common p-belief in AI and cognitive modeling.  In particular, we
developed an efficient algorithm for the inference of maximal
perceived common p-belief, we showed that the coordination strategy of
probability matching on common p-belief explains certain surprising
qualitative features of existing data from a previous human
experiment, and we showed that this model may also help improve agent
outcomes in human-agent coordination.  This work has three main
limitations.  Due to the small amount of human data we had and the
lack of a held-out test set, our empirical results are necessarily
only suggestive.  While the data are inconsistent with the rational
p-belief model and the iterated maximization model, the predictions of
the iterated matching model and the matched p-belief model are both
reasonably good.  The strongest evidence we have favoring the matched
p-belief model is this model's ability to produce equal amounts of
coordination in the secondary and tertiary knowledge conditions as
well as a low amount with private knowledge and a high amount with
common knowledge.  No iterated reasoning model under any formulation
we could find of the Thomas experiments was able to capture the
equality between the two middle conditions while maintaining good
predictions at the extremes.  Two other important limitations of our
work are that the coordination task we consider did not involve
intentional communication, and that the state and action spaces of the
task were simple.  While these features allowed us to easily test the
predictions of each of our alternative models, it would be interesting
to see how the models we considered would compare in more complex
environments.  A related interesting direction for future work is the
application of inference of common p-belief through reasoning about
p-evident events to artificial distributed systems, such as for
developing or analyzing bitcoin/blockchain-like protocols,
synchronizing remote servers, or distributed planning in ad hoc
multi-robot teams.

\section{Acknowledgments}

Special thanks to Kyle Thomas for providing data and Moshe Hoffman for
bringing our attention to p-evident events.  This material is based
upon work supported by the NSF GRFP, grant \#1122374; the Center for
Brains, Minds \& Machines (CBMM), under NSF STC award CCF-1231216; and
by NSF grant IIS-1227495 and ARO grant \#6928195.  Any opinion,
findings, and conclusions or recommendations expressed in this
material are those of the authors and do not necessarily reflect the
views of the sponsors.

\bibliographylatex{commonknowledge}
\bibliographystylelatex{aaai}

\pagebreak

\title{Modeling Human Ad Hoc Coordination}

\author{Peter M. Krafft$^*$, Chris L. Baker$^\dagger$, Alex ``Sandy'' Pentland$^\ddagger$, Joshua B. Tenenbaum$^\dagger$\\Massachusetts Institute of Technology, Cambridge, MA USA\\$^*$Computer Science and Artificial Intelligence Laboratory, $^\dagger$MIT Media Lab, $^\ddagger$Department of Brain and Cognitive Sciences\\\{pkrafft,clbaker,pentland,jbt\}@mit.edu}

\maketitle

\section*{Supplemental Materials}

\setcounter{equation}{0}
\setcounter{figure}{0}
\setcounter{table}{0}
\setcounter{page}{1}
\setcounter{section}{0}
\makeatletter
\renewcommand{\thepage}{S\arabic{page}}
\renewcommand{\thesection}{S\arabic{section}}
\renewcommand{\thetable}{S\arabic{table}}
\renewcommand{\theequation}{S\arabic{equation}}
\renewcommand{\thefigure}{S\arabic{figure}}

These supplementary sections present the proofs of the correctness of
our algorithms.  Our key proofs are self-contained, but they rely on
some preliminary results involving the structure of p-evident events
in finite probability spaces.  We first offer proofs of these
preliminary results in the ``Initial Results'' section.  We then
prove that the rational p-belief strategy is an equilibrium in the
``Models'' section.  Finally, we show the correctness of our algorithm
for computing perceived maximal common p-belief in the
``Algorithms'' section.  This final section also includes our
algorithm for computing information partitions from a given
probabilistic generative world model.


All of the proofs are original.  Some of the less trivial lemmas and
propositions may be novel contributions to the literature on
p-evident events.  In particular, we know of no work that explores
the structure of what we will call ``maximally evident $C$-indicating
events'', i.e. maximally p-evident events in which $C$ is common
knowledge at a given state $\omega$.  Our main result along these
lines will be that any finite state space has a unique representation
as a nested sequence of maximally evident $C$-indicating events.  This
result is central to understanding our algorithm for computing common
p-belief.


\section{Definitions}

We first restate the most relevant definitions from our main text in a
clearer format, and we give several additional definitions that will
be helpful in our proofs.  As described in the main text, we assume a
Bayesian game with a finite state space $\Omega = \{\omega_1, \ldots,
\omega_{|\Omega|}\}$ and information partitions $\Pi_0$ and $\Pi_1$,
which are simply partitions of the state space $\Omega$.  The notation
$\Pi_i(\omega)$ indicates the unique element of $\Pi_i$ that includes
the state $\omega$.  These information partitions and a measure $\mu$
defined over $\Omega$ induce the conditional distribution $P_i(E \g
\omega) = \mu(E \g \Pi_i(\omega)) = \frac{\sum_{\omega' \in E \cap
    \Pi_i(\omega)} \mu(\omega')}{\sum_{\omega' \in
    \Pi_i(\omega)}\mu(\omega')}$, which specifies the belief about an
event $E$ that the player $i$ holds at state $\omega$.

\begin{definition}[\cite{monderer1989approximating}]
Player $i$ is said to \textbf{$p$-believe} an event $E \subseteq
\Omega$ at $\omega$ if $P_i(E \g \omega) \geq p$. 
\end{definition}

\begin{definition}[\cite{monderer1989approximating}]
An event $E$ is \textbf{$p$-evident} if for each $\omega \in E$, all
players $p$-believe $E$ at $\omega$.
\end{definition}

\begin{definition}
An event $E$ is \textbf{super-$p$-evident} if
for all $\omega \in E$ and for all players $i$, $P_i(E \g \omega) > p$.
\end{definition}

\begin{definition}
An event $E$ is said to be a \textbf{$p$-evident $C$-indicating event}
if $E$ is a $p$-evident event and if $P_i(C \g \omega) \geq p$ for all
$i$ and for all $\omega \in E$.
\end{definition}

\begin{definition}[\cite{monderer1989approximating}]
The players have \textbf{common $p$-belief} in an event $C$ at state
$\omega$ if there exists a $p$-evident $C$-indicating event $E$ that
includes $\omega$.
\end{definition}

\begin{definition}
An event $E$ is the \textbf{largest $p$-evident $C$-indicating event}
(for a particular $p$) if $E$ is a $p$-evident $C$-indicating event
and for all events $F$, either $F$ is not a $p$-evident $C$-indicating
event, or $F \subseteq E$.
\end{definition}


\begin{definition}
The \textbf{$C$-evidence level} of an event $E$ is the maximum value of
$p$ for which $E$ is a $p$-evident $C$-indicating event.
\end{definition}

\begin{definition}
An event $E$ is the \textbf{maximally evident $C$-indicating superset}
of $F$ if there exists a $p$ such that $p$ is the $C$-evidence level
of $E$, $E$ is the largest $p$-evident $C$-indicating event, and for
all other events $G$, either $G$ does not contain $F$ or $G$ has a
lower $C$-evidence level than $E$.
\end{definition}

\begin{definition}
An event $E$ is the \textbf{largest super-evident $C$-indicating
  subset} of $F$ if there exists a $p$ such that $p$ is the
$C$-evidence level of $F$, $E$ is a super-$p$-evident $C$-indicating
event, and for all other events $G$, either $G$ is not a subset of
$F$, $G$ is a subset of $E$, or the $C$-evidence level of $G$ is less
than or equal to $p$.
\end{definition}


\section{Initial Results}


Our first result is a simple lemma showing that the family of
$p$-evident $C$-indicating events is closed under unions.

\begin{lemma}
  \label{lemma:union}
For a given event $C$, if $E$ and $F$ are $p$-evident $C$-indicating
events, then $E \cup F$ is also a $p$-evident $C$-indicating event.
\end{lemma}

\begin{proof}
Consider $\omega \in E \cup F$.  For any player $i$, $P_i(E \cup F \g
\omega) \geq \min(P_i(E \g \omega), P_i(F \g \omega)) \geq p$ (since
$E$ and $F$ are $p$-evident).  Therefore $E \cup F$ is $p$-evident.
Now suppose $\omega \in E$. In this case $P_i(C \g \omega) \geq p$
since $E$ is $C$-indicating.  But also, if $\omega \in F$, then $P_i(C
\g \omega) \geq p$ since $F$ is $C$-indicating.  Therefore $E \cup F$
is $C$-indicating.
\end{proof}

Our next result is an immediate corollary and establishes the
existence and uniqueness of largest $p$-evident $C$-indicating events
for any $p$.

\begin{corollary}
  \label{cor:uniqueness}
  For a given event $C$, if there exists a $p$-evident $C$-indicating
  event $E$, then there exists a unique largest $p$-evident
  $C$-indicating event.
\end{corollary}

\begin{proof}
  Assume there exists a $p$-evident $C$-indicating event.  Take $\mF$
  to be the set of all $p$-evident $C$-indicating events.  By lemma
  \ref{lemma:union}, the set $G = \cup_{F \in \mF} F$ consisting of
  the union of all of these events is also a $p$-evident
  $C$-indicating event.  Therefore $G$ is a $p$-evident $C$-indicating
  event containing all other $p$-evident $C$-indicating events.
  Moreover, $E$ must be unique since it contains all other $p$-evident
  events that indicate $C$.
\end{proof}

The next lemma establishes a containment relationship between largest
$p$-evident $C$-indicating events associated with different values of
$p$.

\begin{lemma}
  \label{lemma:containment}
  For a given event $C$, if $E$ is the largest $p$-evident
  $C$-indicating event, and $F$ is the largest $p'$-evident
  $C$-indicating event, with $p \geq p'$, then $E \subseteq F$.
\end{lemma}

\begin{proof}
Suppose $E$ is the largest $p$-evident $C$-indicating event.  Since
$p' \leq p$, $E$ is also a $p'$-evident $C$-indicating event.  Thus $E$
must be contained by the largest $p'$-evident $C$-indicating event.
\end{proof}

We now show that a unique maximally evident $C$-indicating event
exists around any subset of $\Omega$.

\begin{lemma}
  \label{lemma:evidentiality}
  For a given event $C$, and for any event $E$, there exists a unique
  maximally evident $C$-indicating superset of $E$.
\end{lemma}

\begin{proof}
  Consider an event $E$.  The $C$-evidence level of $E$ is given by $p
  = \min_{i \in [0,1],\omega \in E} \min(P_i(E \g \omega), P_i(C \g
  \omega))$, and therefore exists.  Moreover, since $E$ is a
  $p$-evident $C$-indicating event, $E$ must be contained by the
  unique largest $p$-evident $C$-indicating event $F$ (which is
  guaranteed to exist by corollary \ref{cor:uniqueness}).  Further, by
  definition $F$ must also contain all other $p$-evident supersets of
  $E$, and therefore $F$ is the maximally evident $C$-indicating
  superset of $E$.
\end{proof}

The following lemma, which will be useful in proving our main
theoretical result, shows that the family of maximally evident
$C$-indicating events is highly constrained.

\begin{lemma}
  \label{lemma:reduction}
  For a given event $C$, for any event $E$, there exists some $\omega
  \in E$ such that the maximally evident $C$-indicating superset of
  $E$ is equal to the maximally evident $C$-indicating superset of
  $\{\omega\}$.
\end{lemma}

\begin{proof}
  Let $F$ be the maximally evident $C$-indicating superset of $E$,
  guaranteed to exist by lemma \ref{lemma:evidentiality}.  Let $p_F$
  be the $C$-evidence level of $F$.  Consider an $\omega \in E$.  Let
  $G_\omega$ be the maximally evident $C$-indicating superset of
  $\{\omega\}$, and let $p_{G_\omega}$ be the $C$-evidence level of
  $G_\omega$.  We must have that $p_{G_\omega} \geq p_F$ for all
  $\omega$ (Since $F$ contains $\omega$, $p_F > p_{G_\omega}$ would
  violate the fact that $G_\omega$ is maximally evident.)  Suppose
  $p_{G_\omega} > p_F$ for all $\omega$.  Then, by lemma
  \ref{lemma:union}, $H = \cup_{\omega \in E} G_\omega$ would be a
  $p_H$-evident $C$-indicating event with $p_H \geq \min_\omega
  p_{G_\omega} > p_F$.  However, $H$ clearly contains $E$, and thus
  $p_H > p_F$ violates the fact that $F$ is the maximally evident
  $C$-indicating superset of $E$.  Therefore, there must exist some
  $\omega$ such that $p_{G_\omega} = p_F$.  Since largest p-evident
  $C$-indicating events are unique by corollary \ref{cor:uniqueness},
  we must have $H = G_\omega$.
\end{proof}

\begin{figure*}
  \centering
  \includegraphics[width=0.5\linewidth]{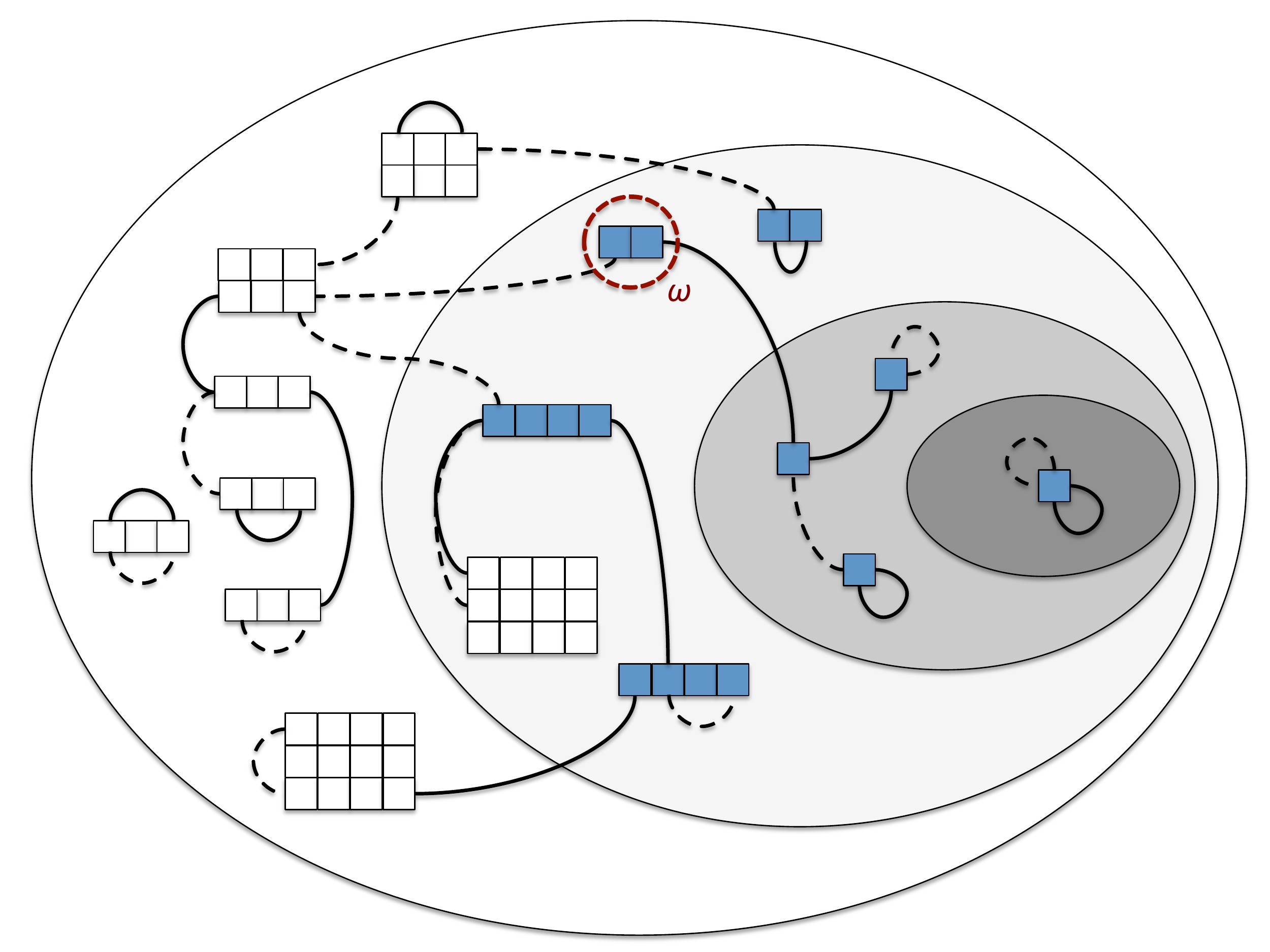}
  \caption{Any finite state space can be uniquely represented as a
    nested sequence of maximally evident $C$-indicating events.  The
    diagram in this figure represents the generative process of the
    messenger described in the main text of our paper (with $\delta =
    0.25$).  Each contiguous rectangle of blocks represents a state in
    $\Omega$, and the measure of the state is given by the area of the
    rectangle.  States that are shaded are members of $C = \{\omega
    \in \Omega : x(\omega) = 1\}$.  The solid lines between states
    represent the information partition of player 0, while the the
    dotted lines represent that of player 1.  Two states belong to the
    same element of a player's information partition if they are
    connected by some path in the graph induced by that player's
    edges.  Self-loops indicate that a player has no uncertainty about
    the state when the state obtains.  The four nested ovals are the
    four maximally evident $C$-indicating events in this state space
    ($\mF$ in theorem \ref{thm:nesting}), and the grey shading in the
    ovals represents the $C$-evidence levels of those events (0, 0.25,
    0.5, and 1.0).  Our algorithm iterates over these maximally
    evident events rather than over all possible events, and at
    $\omega$ for player $i$ returns the $C$-evidence level associated
    with the last such event containing some element of
    $\Pi_i(\omega)$.  For instance, at the circled state the algorithm
    will find the third nested event for player 0 and the second for
    player 1.}
  \label{fig:nesting}
\end{figure*}

The following theorem is our main theoretical result, and drives the
correctness of our algorithms.  This theorem, which is also
illustrated in Figure \ref{fig:nesting}, states that any finite
$\Omega$ has a unique representation as a nested sequence of maximally
evident $C$-indicating events.  The efficiency of our algorithm stems
from only searching through this sequence of subsets, rather than all
possible subsets, in order to compute common p-belief at any state
$\omega$.


\begin{theorem}
  \label{thm:nesting}
  Given an information structure $(\Omega, \mu, (\Pi_0,\Pi_1))$ with
  $|\Omega| < \infty$ and any event $C \subseteq \Omega$, let $\mF =
  \{E_1',\ldots,E_n'\}$ be the set of maximally evident $C$-indicating
  events of $\Omega$, i.e. the set of events $E_i'$ for which there
  exists some $F \subseteq \Omega$ such that $E_i'$ is the maximally
  evident $C$-indicating superset of $F$.  $\mF$ can be ordered into a
  nested sequence of subsets, $E_1 \supset E_2 \supset \ldots \supset
  E_n$ such that $E_{i}$ is the largest super-evident $C$-indicating
  subset event of $E_{i-1}$ for any $i > 1$.
\end{theorem}

\begin{proof}
We prove this theorem by construction.  We know from lemma
\ref{lemma:evidentiality} that for each $\omega_j \in \Omega$, there
exists a maximally evident $C$-indicating superset of
$\{\omega_j\}$. We label this event $E'_{\omega_j}$.  Lemma
\ref{lemma:reduction} implies that collection of events
$\{E'_{\omega_j} : \omega_j \in \Omega\}$ is equal to $\mF$, the
entire set of maximally evident $C$-indicating events.  We also know
from lemma \ref{lemma:evidentiality} that each $E'_{\omega_j}$ is
associated with a particular $C$-evidence level, which we will label
$p_{\omega_j}$.  Without loss of generality, we can assume that the
index $j$ sorts these $C$-evidence levels by their magnitudes.  This
produces a finite non-decreasing sequence of values in $[0,1]$:
$p_{\omega_1} \leq p_{\omega_2} \leq \ldots \leq
p_{\omega_{|\Omega|}}$. By lemma \ref{lemma:containment} the events
$E'_{\omega_j}$ thus form a nested sequence of subsets: $E'_{\omega_1}
\supseteq E'_{\omega_2} \supseteq \ldots \supseteq
E'_{\omega_{|\Omega|}}$.  We can collapse the events that are equal to
each other to arrive at a sequence of strict subsets: $E_1 \supset E_2
\supset \ldots \supset E_n$.

Now consider $E_i$ for some $i > 1$. By construction $E_i$ is a
super-$p_{\omega_{i-1}}$-evident $C$-indicating event, and $E_i
\subset E_{i-1}$.  Now take $F$ to be another
super-$p_{\omega_{i-1}}$-evident $C$-indicating event, and let $p_F$
be the $C$-evidence level of $F$.  If $p_F \geq p_{\omega_i}$, then we
must have $E_i \supseteq F$ by lemma \ref{lemma:containment} since
$E_i$ is the largest $p_{\omega_i}$-evident event.  Now suppose $p_F <
p_{\omega_i}$, so $p_{\omega_{i - 1}} < p_F < p_{\omega_i}$.  Let $G$
be the maximally evident $C$-indicating superset of $F$, and let $p_G$
be the $C$-evidence level of $G$.  Clearly $p_G \geq p_F$, and by
construction we must also have $p_G = p_{\omega_j}$ since $\mF$
contains all maximally evident $C$-indicating supersets, and hence $G
= E_j$, for some $j$.  But since $p_{\omega_j} = p_G \geq p_F >
p_{\omega_{i-1}}$, we must have $j \geq i$, and so $F \subseteq G =
E_j \subseteq E_i$.  Therefore again $F \subseteq E_i$.  Hence $E_i$
must be the largest $C$-indicating subset of $E_{i-1}$.
\end{proof}

The following corollary to this theorem provides the basis for our
iterative algorithm.

\begin{corollary}
  \label{cor:next}
  Given a maximally evident $C$-indicating event $E$, either there
  exists a unique largest super-evident $C$-indicating subset of $E$,
  or no super-$p$-evident $C$-indicating subsets exist.
\end{corollary}

\begin{proof}
  Since $E$ is a maximally evident $C$-indicating event, $E$ is equal
  to some $E_i$ in the sequence $\mF$ given by theorem
  \ref{thm:nesting}.  Therefore, if $i < n$, $E_{i+1}$ is the unique
  largest super-evident $C$-indicating subset of $E$ or no such
  subsets exist, as indicated in theorem \ref{thm:nesting}.
\end{proof}


Finally, we have four technical lemmas that will be useful for
analyzing our specific algorithms and player strategies.  The first
lemma is a simple constraint on belief about p-evident events that
follows from how information partitions work.

\begin{lemma}
  \label{lemma:zero}
  For any $p$-evident event $E$, either $P_i(E \g \omega) = 0$ or
  $P_i(E \g \omega) \geq p$ for all $\omega$ and for all $i$.
\end{lemma}

\begin{proof}
  If $P_i(E \g \omega) > 0$, then there must exist an $\omega' \in
  \Pi(\omega)$ such that $\omega' \in E$.  Since $E$ is $p$-evident,
  then $P_i(E \g \omega') \geq p$. But then since $\Pi_i(\omega) =
  \Pi_i(\omega')$ by the definition of an information partition, $P_i(E \g
  \omega) = \mu(E \g \Pi_i(\omega)) = \mu(E \g \Pi_i(\omega')) = P_i(E \g
  \omega')$, and hence $P_i(E \g \omega) \geq p$.
\end{proof}

The second lemma connects the containment relationships between
p-evident events to the relationship between beliefs about those
events.

\begin{lemma}
  \label{lemma:belief}
  If player $i$ $p$-believes a $p$-evident $C$-indicating event, then
  player $i$ $p$-believes the largest $p$-evident $C$-indicating
  event.
\end{lemma}

\begin{proof}
  Let $F$ be a $p$-evident $C$-indicating event and let $E$ be the
  largest $p$-evident $C$-indicating event.  Since $E$ is largest, $F
  \subseteq E$, so we must have $P_i(E \g \omega) = P_i(E \cup F \g
  \omega) \geq P_i(F \g \omega)$.  Hence if $i$ $p$-believes $F$, $i$
  must also $p$-believe $E$.
\end{proof}

The next lemma states that p-belief is transitive.

\begin{lemma}
  \label{lemma:trans}
  If player $i$ $p$-believes $F$ at all elements of an event $E$, and
  player $i$ $p$-believes $G$ at all elements of $F$, then player $i$
  $p$-believes $G$ at all elements of $E$.
\end{lemma}

\begin{proof}
  Let $\omega$ be an element of E. Since $i$ $p$-believes $F$ at
  $\omega$, $P_i(F \g w) \geq p$. Therefore we must have some $\omega'
  \in \Pi_i(\omega) \cap F$. Since player $i$ $p$-believes $G$ at all
  states in $F$, we must have $P_i(G \g \omega') \geq p$.  By the
  definition of an information partition, we must then also have
  $P_i(G \g \omega) = P_i(G \g \omega') \geq p$.  Hence player $i$
  $p$-believes $G$ at all $\omega \in E$.
\end{proof}

The final lemma states, roughly, that with regard to p-evident
events a player's beliefs about another player cannot be too
inconsistent with the first player's own beliefs.

\begin{lemma}
  \label{lemma:none}
  If player $i$ $p$-believes that player $1 - i$ $p$-believes the
  largest $p$-evident $C$-indicating event $E$, and player $i$
  $p$-believes $C$, then player $i$ must $p$-believe $E$.
\end{lemma}

\begin{proof}
  Consider an arbitrary $\omega \in \Omega$ and an arbitrary $p$.
  Take the largest $p$-evident $C$-indicating event $E$.  Suppose at
  $\omega$ player $i$ $p$-believes that player $1 - i$ $p$-believes
  $E$, and player $i$ $p$-believes $C$.  To generate the first belief,
  there must be an element of the information partition of player $1 -
  i$, $T \in \Pi_{1 - i}$, such that $P_{1- i}(E \g T) \geq p$, and
  another set of states $S \subseteq \Pi_{i}(\omega) \cap T$ such that
  $P_i(S \g \omega) \geq p$.  But then, since states within a player's
  information partition are indistinguishable and since $S \subseteq
  \Pi_{i}(\omega) \cap T$, we must have for all $\omega' \in S$, $P_{1
    - i}(E \g \omega') = P_{1-i}(E \g T) \geq p$ and $P_i(S \g
  \omega') = P_i(S \g \omega) \geq p$.  Therefore, since $E$ is
  $p$-evident, $S \cup E$ is also a $p$-evident event (if $\omega''
  \in S$, then player $1 - i$ $p$-believes $E$ and player $i$
  $p$-believes $S$, while if $\omega'' \in E$, then both players
  $p$-believe $E$).  Moreover, since player $i$ $p$-believes $C$ at
  $\omega$ (and therefore at $\Pi(\omega)$ and $S$) and since $1 - i$
  $p$-believes $C$ at $E$ (and hence at $S$ by lemma
  \ref{lemma:trans}), then both players $p$-believe $C$ at all states
  in $S \cup E$.  Therefore $S \cup E$ is a $p$-evident $C$-indicating
  event (and, moreover, player $i$ $p$-believes $S \cup E$ at $\omega$
  since player $i$ $p$-believes $S$ at $\omega$).  But $E$ was assumed
  to be the largest $p$-evident $C$-indicating event, so we must have
  $S \cup E = E$. Hence player $i$ $p$-believes $E$.
\end{proof}


\section{Models}

\subsection*{Strategic Coordination}

In this section we show the rational p-belief strategy forms an
equilibrium in the coordination game we study.  Recall the rational
p-belief strategy is that player $i$ plays action $A$ if and only if
player $i$ believes with at least probability $p^* = \frac{c - b}{a -
  b}$ that both players have common $p^*$-belief that $x = 1$.

\begin{proposition}
Assuming noiseless messages and assuming $p^* > \mu(x = 1)$, the
rational p-belief strategy maximizes the expected return of player $i$
at every $\omega \in \Omega$, given player $1 - i$ also uses the same
strategy.
\end{proposition}

\begin{proof}
 Suppose that messages are noiseless, i.e. $P_i(x = 1 \g \omega) >
 P_i(x = 1) \Rightarrow P_i(x = 1 \g \omega) = 1$ for all $\omega$.
 Take $\omega \in \Omega$.

 Take $F$ to be some $p^*$-evident ($x = 1$)-indicating event.  Take
 $E$ to be the largest such event.  Suppose player $i$ $p^*$-believes
 $F$. By lemma \ref{lemma:belief}, player $i$ also $p^*$-believes $E$.
 By the definition of p-evident events, for any $\omega' \in E$, we
 must then have that player $1 - i$ also $p^*$-believes $E$ at
 $\omega'$.  Therefore by lemma \ref{lemma:trans} player $i$ also
 $p^*$-believes that player $1 - i$ $p^*$-believes $E$.  Since player
 $1 - i$ is assumed to be using the rational p-belief strategy, player
 $i$ therefore $p^*$-believes player $1 - i$ will play $A$.  Since we
 have assumed $p^* > P_i(x = 1)$ and messages are noiseless, we must
 have $P_i(x = 1 \g \Pi_i(\omega)) = 1$.  Then the expected return for
 $i$ of playing $A$ must be at least $p^* \cdot a + (1 - p^*)b =
 \frac{c - b}{a - b} a + \frac{a - c}{a - b}b = c$.  Therefore,
 playing $A$ maximizes player $i$'s expected return (and if player $i$
 $p'$-believes a $p'$-evident $(x = 1)$-indicating event, $p' > p^*$,
 then the expected return from $A$ is strictly greater than for $B$).

 Now suppose player $i$ $p^*$-believes $x = 1$ but that there is no
 $p^*$-evident ($x = 1$)-indicating event $F$ such that $P_i(F \g
 \omega) \geq p^*$.  Take $E$ to be the largest $p^*$-evident $(x =
 1)$-indicating event.  Then player $i$ cannot $p^*$-believe that
 player $1 - i$ $p^*$-believes $E$ (since otherwise player $i$ would
 $p^*$-believe $E$ by lemma \ref{lemma:none}).  Thus, since player $1
 - i$ is assumed to be using the rational p-belief strategy, player
 $i$ believes with probability at most $p' < p^*$ that player $1 - i$
 will play $A$.  Then since $p' \cdot a + (1 - p')b < p^* \cdot a + (1
 - p^*)b = c$, the expected utility from playing $A$ to player $i$ is
 less than from playing $B$. (And if $i$ does not $p^*$-believe $x =
 1$, the same arithmetic holds.)
 
\end{proof}

\section{Algorithms}

\subsection*{Computing Information Partitions}

In this section we state the simple algorithm for converting the
generative process descriptions of the probabilistic models in our
main text to information partitions.  The ``run($i$, $\omega$)''
function takes as input a player $i$ and a state $\omega$ expressed as
a tuple, uses the variables in the tuple for each random draw within
the generative process, and returns a composition of the $observe$
calls for player $i$ (For example, run(0, (1,1,0,1,0)) would return
[(1),(0,0)] in the messenger model).  Our algorithm groups together
states with the same observations into elements of the information
partitions.

\begin{algorithm}
  \caption{information\_partition($i$)}
\begin{algorithmic}
\State partition $:=$ dict()
\For{$\omega \in \Omega$}
\State obs $:=$ run($i$, $\omega$)
\If{obs $\in$ partition}
\State partitions[obs].append($\omega$)
    \Else
      \State partitions[obs] $:=$ [$\omega$]
    \EndIf
\EndFor  
\State \Return partitions
\end{algorithmic}
\end{algorithm}

\subsection*{Computing Common p-Belief}

We now prove the correctness of the common\_p\_belief algorithm given
in the main text.  We first prove a lemma stating the correctness of
the super\_p\_evident algorithm.

\begin{lemma}
  \label{lemma:super}
  Assuming $E$ is a maximally evident $C$-indicating event, and
  assuming $p$ is the $C$-evidence level of $E$, then the set returned
  by evaluating ``super\_p\_evident($E, p$)'' is the largest
  super-evident $C$-indicating subset of $E$, or the empty set if and
  only if such an event does not exist.
\end{lemma}

\begin{proof}
First note that from the condition in the ``if'' statement of the
function, all states that remain in $E$ when the function returns will
have the properties $P_i(E \g \omega) > p$ and $P_i(C \g \omega) > p$.
Therefore, the function only returns super-$p$-evident $C$-indicating
subsets of $E$ (and since $p$ is the $C$-evidence level of $E$, strict
subsets must be returned), or the empty set if no super-$p$-evident
subsets exists.  Next, note that the ``while'' loop only removes
elements of $E$ that cannot belong to the largest super-evident
$C$-indicating subset of $E$, if such an event exists.  We can see
this by induction on the items removed.  Suppose that the first
$\omega$ removed belonged to $F$, the largest super-evident
$C$-indicating subset of $E$.  Then we would have $P_i(F \g \omega) >
p$ since $F$ must be super-$p$-evident.  But since $F \subseteq E$, by
lemma \ref{lemma:containment} we would then also have $P_i(E \g
\omega) > p$, which contradicts the fact that $\omega$ was removed.
Now assume that the first $k$ elements removed do not belong to $F$,
and let $E'$ be the set $E$ minus those elements.  Then suppose the
$(k+1)$st element removed, if it exists, belonged to $F$.  Since all
the previous elements removed did not belong to $F$ by the inductive
assumption, $F \subseteq E'$.  But then analogous to the base case,
the $(k+1)$st element, $\omega_{k+1}$, belonging to $F$ implies
$P_i(E' \g \omega_{k+1}) > p$, which contradicts the fact that this
element was removed.  Therefore this function returns exactly the
largest super-evident $C$-indicating subset of $E$.  Finally, since
$\Omega$ is assumed to be finite, and at least one element of $E$ is
removed in each iteration of the while loop, the function must
terminate.
\end{proof}

Lastly, we show the correctness of our main algorithm.

\begin{proposition}
  \label{prop:correctness}
The value returned by evaluating ``common\_p\_belief($C,i,\omega$)''
is the $C$-evidence level of the maximally evident $C$-indicating
event $E$ containing some element of $\Pi_i(\omega)$.
\end{proposition}

\begin{proof}
First note ``evidence\_level(F, C)'' computes the $C$-evidence level
of $F$ (which must exist by lemma \ref{lemma:evidentiality}).  Also
note $\Omega$ is the maximally evident $C$-indicating superset of
$\Omega$.  Thus by lemma \ref{lemma:super}, the calls to
``super\_p\_evident()'' iteratively return the nested sequence of
subsets of maximally evident subsets described in theorem
\ref{thm:nesting}.

If there does not exist a $p$-evident $C$-indicating event for $p >
0$, then by lemma \ref{lemma:super} the first call to
super\_p\_evident will return the empty set.  In this case the
$C$-evidence level of $\Omega$ will be 0, and hence the function will
return 0.

Since the ``while'' loop will continue iterating until either $F$ is
empty or $F$ is an event that player $i$ believes to be impossible
(i.e., that does not contain any elements in $\Pi_i(\omega)$), the
last $E$ before either of these cases occurred must be the maximally
evident $C$-indicating event containing some element of
$\Pi_i(\omega)$.  The call to evidence\_level then computes the
$C$-evidence level of $E$.
\end{proof}

Our final result interprets the last result in terms of
common-p-belief.

\begin{corollary}
The value returned by evaluating ``common\_p\_belief($C,i,\omega$)''
is the maximum value of $p$ such that player $i$ $p$-believes there is
common-$p$-belief in $C$ at $\omega$.
\end{corollary}

\begin{proof}
  By proposition \ref{prop:correctness},
  common\_p\_belief($C,i,\omega$) returns the $C$-evidence level of
  some event $E$ that player $i$ believes with probability greater
  than 0.  Let $p$ equal the $C$-evidence level of $E$. By lemma
  \ref{lemma:zero}, player $i$ must therefore $p$-believe $E$.  Player
  $i$ thus $p$-believes there is common-$p$-belief in $C$ at $\omega$
  since $E$ is a $p$-evident $C$-indicating event.  Now suppose there
  was some $p' > p$ such that player $i$ $p'$-believes that there is
  common-$p'$-belief in $C$.  Then by definition of common-p-belief
  there must exist a $p'$-evident event $F$ that player $i$
  $p'$-believes.  But then player $i$ must also $p'$-believe the
  maximally evident $C$-indicating superset of $F$.  However, this
  contradicts the fact that $p$ was returned by
  common\_p\_belief($C,i,\omega$), since if player $i$ $p'$-believes
  $F$, then $F$ is a maximally evident event containing some element
  of $\Pi_i(\omega)$.
\end{proof}

\bibliography{commonknowledge}
\bibliographystyle{aaai}

\end{document}